\documentclass{article}
\usepackage{ijcai23}

\usepackage{times}
\usepackage{soul}
\usepackage{url}
\usepackage[hidelinks]{hyperref}
\usepackage[utf8]{inputenc}
\usepackage[small]{caption}
\usepackage{graphicx}
\usepackage{amsmath}
\usepackage{amsthm}
\usepackage{booktabs}
\usepackage{algorithm}
\usepackage{algorithmic}
\usepackage[switch]{lineno}

\usepackage{bm}
\usepackage{multirow}
\usepackage{threeparttable}
\usepackage{subfigure}
\usepackage{makecell}
\usepackage{hyperref}
\usepackage{xcolor}
\usepackage{bbm}
\usepackage{amssymb}
\usepackage{booktabs}
\usepackage{amsthm}
\usepackage{epstopdf}

\usepackage{threeparttable}

\DeclareMathOperator*{\trace}{tr}

\newtheorem{proposition}{Proposition}

\title{Deep Partial Multi-Label Learning with Graph Disambiguation}

\author{
Haobo Wang $^1$\and
Shisong Yang$^2$\and
Gengyu Lyu$^2$\thanks{Corresponding author.}\and
Weiwei Liu $^3$\and
Tianlei Hu$^1$\and\\
Ke Chen$^1$\and
Songhe Feng$^4$\And
Gang Chen$^1$
\affiliations
$^1$Zhejiang University$\quad$
$^2$Beijing University of Technology\\
$^3$Wuhan University $\quad$
$^4$Beijing Jiaotong University
\emails
\{wanghaobo, htl, chenk, cg\}@zju.edu.cn,
yangshisong@emails.bjut.edu.cn,\\ lyugengyu@bjut.edu.cn,
liuweiwei863@gmail.com,
shfeng@bjtu.edu.cn
}

\begin{document}

\maketitle

\begin{abstract}
    In partial multi-label learning (PML), each data example is equipped with a candidate label set, which consists of multiple ground-truth labels and other false-positive labels. Recently, graph-based methods, which demonstrate a good ability to estimate accurate confidence scores from candidate labels, have been prevalent to deal with PML problems. However, we observe that existing graph-based PML methods typically adopt linear multi-label classifiers and thus fail to achieve superior performance. In this work, we attempt to remove several obstacles for extending them to deep models and propose a novel deep \textbf{P}artial multi-\textbf{L}abel model with gr\textbf{A}ph-disamb\textbf{I}guatio\textbf{N} (\textbf{PLAIN}). Specifically, we introduce the instance-level and label-level similarities to recover label confidences as well as exploit label dependencies. At each training epoch, labels are propagated on the instance and label graphs to produce relatively accurate pseudo-labels; then, we train the deep model to fit the numerical labels. Moreover, we provide a careful analysis of the risk functions to guarantee the robustness of the proposed model. Extensive experiments on various synthetic datasets and three real-world PML datasets demonstrate that PLAIN achieves significantly superior results to state-of-the-art methods. 
\end{abstract}

\section{Introduction}
The rapid growth of deep learning allows us to tackle increasingly sophisticated problems.
Among them, multi-label learning (MLL), which assigns multiple labels to an object, is a fundamental and intriguing task in various real-world applications such as image retrieval \cite{DBLP:journals/tip/LaiYSWY16}, autonomous vehicles \cite{DBLP:journals/tip/ChenZTHZ19}, and sentiment analysis \cite{9490212}. Typically, these tasks require precisely labeled data, which is expensive and time-consuming due to the complicated structure and the high volume of the output space.

To address this issue, plenty of works have studied different weakly-supervised settings of MLL, including semi-supervised MLL \cite{DBLP:conf/kdd/ShiSLG20}, multi-label with missing labels \cite{DBLP:conf/cvpr/DurandMM19,DBLP:conf/cvpr/HuynhE20b}, and single-labeled MLL \cite{DBLP:journals/corr/abs-2206-00517}.
While these learning paradigms mainly battle the rapid growth of data volume, they overlook the intrinsic difficulty and ambiguity in multi-label annotation.
For example, it is hard for a human annotator to identify an Alaskan Malamute from a Siberian Husky.
One potential solution is to randomly decide the binary label value. However, it may involve uncontrollable label noise that has a negative impact on the training procedure. Another strategy is to leave it missing, but then, we may lose critical information. Recently, partial multi-label learning (PML) \cite{DBLP:conf/aaai/XieH18} has been proposed to alleviate this problem. In PML, the annotators are allowed to provide a candidate label set that contains all potentially relevant labels as well as false-positive labels.


A straightforward approach to learning from PML data is to regard all candidate labels as valid ones, and then, off-the-shelf multi-label methods can be applied. Nevertheless, the presence of false-positive labels can highly impact predictive performance. To cope with this issue, many PML methods \cite{DBLP:conf/aaai/XieH18,DBLP:conf/icdm/YuCDWLZW18,DBLP:conf/icdm/HeD0SL19,DBLP:conf/aaai/SunFWLJ19,DBLP:conf/ijcai/LiLF20,DBLP:conf/kdd/LyuFL20,DBLP:conf/icdm/XieH20,DBLP:journals/kbs/YanLF21,DBLP:journals/pami/XieH22} have been developed. Amongst them, graph-based PML methods \cite{DBLP:conf/ijcai/Wang0ZZHC19,DBLP:conf/ijcai/LiW20,DBLP:conf/aaai/0009LG20,DBLP:conf/kdd/LyuFL20,zml-pami-pml} have been popular. In practice, graph-based algorithms demonstrate a good ability to obtain relatively precise confidence of candidate labels by aggregating information from the nearest neighbors. However, we observe that most of them adopt linear classifier, which is less powerful to deal with complicated tasks and high data volumes.

Consequently, there is an urgent need for developing an effective deep model to boost the performance of graph PML methods. However, there are two main challenging issues. First, existing graph-based PML methods typically require propagating labels over the whole training dataset. It restricts their scalability when employing deep classifiers since the deep models are trained in batch mode. Second, it is well acknowledged that there exist label correlations in multi-label datasets, which also helps disambiguate the candidate labels. For example, if labels \textit{Amusement Park} and \textit{Mickey Mouse} are identified as true, it is very likely that label \textit{Disney} is also a ground-truth label. Thus, it is required to exploit the label correlations to improve the disambiguation ability in a unified framework.

In this work, we propose a novel deep \textbf{P}artial multi-\textbf{L}abel model with gr\textbf{A}ph disamb\textbf{I}guatio\textbf{N} (PLAIN). In PLAIN, we involve both instance- and label-level similarities to disambiguate the candidate labels set, which simultaneously leverages instance smoothness as well as label correlations. Then, we build a deep multi-label classifier to fit the relatively credible pseudo-labels generated through disambiguation. Instead of separately training in two stages like previous works \cite{DBLP:conf/ijcai/Wang0ZZHC19,zml-pami-pml}, we propose an efficient optimization framework that iteratively propagates the labels and trains the deep model. Moreover, we give a careful analysis of the risk functions such that proper risk functions are chosen for improved robustness. 
Empirical results on nine synthetic datasets and three real-world PML datasets clearly verify the efficiency and effectiveness of our method. More theoretical and empirical results can be found in Appendix. 






\section{Related Works}

\paragraph{Multi-Label Learning (MLL)} \cite{DBLP:journals/pami/LiuWST22,DBLP:journals/corr/abs-2210-03968} aims at assigning each data example multiple binary labels simultaneously. An intuitive solution, the one-versus-all method (OVA) \cite{DBLP:journals/tkde/ZhangZ14}, is to decompose the original problem into a series of single-labeled classification problems. However, it overlooks the rich semantic information and dependencies among labels, and thus fails to obtain satisfactory performance. To solve this problem, many MLL approaches are proposed, such as embedding-based methods \cite{DBLP:conf/icml/Yu0KD14,DBLP:journals/tnn/ShenLTSO18}, classifier chains \cite{DBLP:journals/ml/ReadPHF11}, tree-based algorithms \cite{DBLP:conf/kdd/WeiSL21}, and deep MLL models \cite{DBLP:conf/aaai/YehWKW17,DBLP:conf/cvpr/ChenWWG19,DBLP:conf/ijcai/BaiKG20}. Nevertheless, these methods typically require a large amount of fully-supervised training data, which is expensive and time-consuming in MLL tasks. To this end, some works have studied different weakly-supervised MLL settings. For example, semi-supervised MLL \cite{DBLP:conf/nips/NiuHSC19,DBLP:conf/kdd/ShiSLG20} leverages information from both fully-labeled and unlabeled data. Multi-label with missing labels  \cite{DBLP:conf/cvpr/DurandMM19,DBLP:journals/ml/WeiGLG18,DBLP:journals/corr/abs-2206-00517} allows providing a subset of ground-truth. In noisy MLL \cite{DBLP:conf/cvpr/VeitACKGB17}, the binary labels may be flipped to the wrong one. Nevertheless, though these settings relieve the burden of multi-label annotating, they ignore the inherent difficulty of labeling, i.e. the labels can be ambiguous.




\paragraph{Partial Multi-Label Learning (PML)} learns from a superset of ground-truth labels, where multiple labels are possible to be relevant. To tackle the ambiguity in the labels, the pioneering work \cite{DBLP:conf/aaai/XieH18} estimates the confidence of each candidate label being correct via minimizing confidence-weighted ranking loss. Some works \cite{DBLP:conf/icdm/YuCDWLZW18,DBLP:conf/aaai/SunFWLJ19,DBLP:conf/ijcai/LiLF20,DBLP:conf/aaai/GongYB22} recover the ground-truth by assuming the true label matrix is low-rank. Recently, graph-based methods \cite{DBLP:conf/ijcai/Wang0ZZHC19,DBLP:conf/aaai/0009LG20,DBLP:conf/kdd/LyuFL20,DBLP:conf/icdm/XieH20,zml-pami-pml,9615493,DBLP:journals/pr/TanLWZ22} have attracted much attention from the community due to the good disambiguation ability. For instance, PARTICLE \cite{zml-pami-pml} identifies the credible labels through an iterative label propagation procedure and then, applies pair-wise ranking techniques to induce a multi-label predictor.
Although graph-based models have shown promising results, we observe that most of them adopt linear classifiers, which significantly limits the model expressiveness. Besides, they typically separate the graph disambiguation and model training into two stages, which makes the multi-label classifier error-prone; in contrast, the proposed PLAIN is an end-to-end framework.

Some studies have also developed deep PML models, such as adversarial PML model \cite{DBLP:journals/corr/abs-1909-06717} and mutual teaching networks \cite{DBLP:journals/kbs/YanLF21}. However, it remains urgent to combine the graph and deep models together such that the model enjoys high expressiveness as well as good disambiguation ability. It should also be noted that the recent popular graph neural networks (GNN) model \cite{DBLP:journals/tnn/WuPCLZY21} is not suitable for PML. The reason is that GNN aims at better representation learning (at an instance-level), while in PML, the graph structure is used to alleviate the label ambiguity.

\section{Method}\label{Section-Models}
We denote the $d$-dimensional feature space as $\mathcal{X}\subset\mathbb{R}_d$ , and the label space as $\mathcal{Y}=\{1,2,\ldots,L\}$ with $L$ class labels. The training dataset $\mathcal{D}=\{(\bm{x}_i, S_i)|1\le i\le n\}$ contains $n$ examples, where $\bm{x}_i\in\mathcal{X}$ is the instance vector and $S_i\subset\mathcal{Y}$ is the candidate label set. Without loss of generality, we assume that the instance vectors are normalized such that $||\bm{x}_i||_2=1$. Besides, we denote $\tilde{S}_i\subset\mathcal{Y}$ as the ground-truth label set, which is a subset of $S_i$. For simplicity, we define $\bm{y}_i$, $\tilde{\bm{y}}_i\in\{0,1\}^L$ as the binary vector representations of $S_i$ and $\tilde{S}_i$. When considering the whole dataset, we also denote the candidate label matrix by $\bm{Y}=[\bm{y}_1,\ldots,\bm{y}_n]^\top\in\mathbb{R}^{n\times L}$.

In what follows, we first introduce the learning target of our model that integrates label correlations to boost disambiguation ability. Then, we provide an efficient optimization framework for our deep propagation network.

\begin{figure*}[t]
\begin{center}
\includegraphics[width=0.75\linewidth]{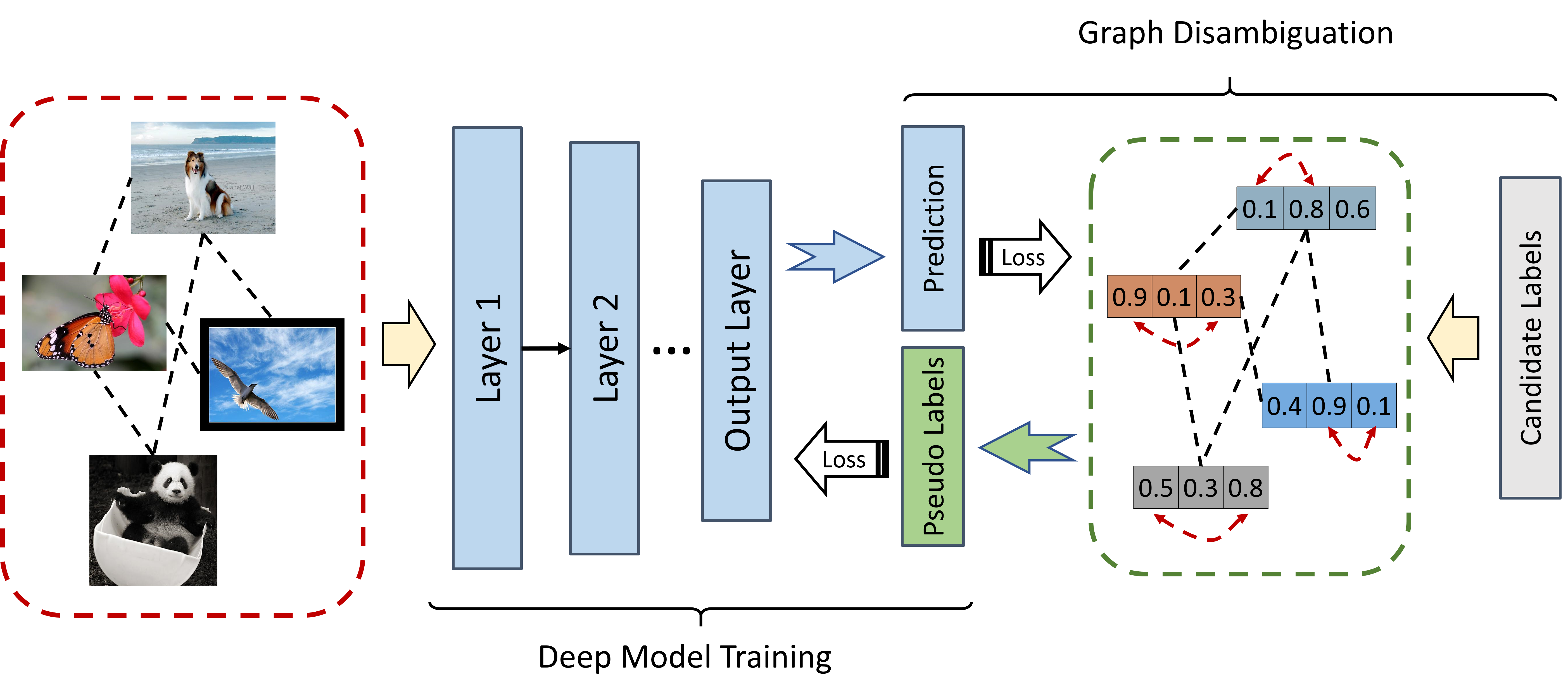}
\end{center}
\vskip -0.1in
   \caption{The model architecture of PLAIN. At each epoch, we first train the deep model by traveling the training dataset and then propagating the labels with the help of candidate labels and model prediction; PLAIN alternates between the two operations. }
\label{fig:model}
\vskip -0.05in
\end{figure*}


\subsection{Learning Objective}
Our ultimate goal is to estimate a multi-label predictor $f:\mathcal{X}\mapsto[0,1]^L$ from $\mathcal{D}$. In this work, we assume that its output $\hat{\bm{y}}=f(\bm{x})$ satisfies the following properties,
\begin{enumerate}
    \item $\hat{\bm{y}}$ is close to the candidate vector $\bm{y}$;
    \item $\hat{\bm{y}}$ satisfies the \textbf{instance-level similarity}. That is, if two instances are close to each other, then their label vectors are likely to be close too;
    \item $\hat{\bm{y}}$ satisfies the \textbf{label-level similarity}. That is, if two labels co-occur in the training data frequently, then they tend to co-exist at many instances. 
\end{enumerate}
These properties coincide with our goals. The first property is natural since the only thing we can trust is the candidate label set. The second one has been widely adopted to disambiguate the ground-truth label from the false-positive labels. The last one exploits the label correlations to further promote the performance.
In this paper, we adopt a graph-based model to achieve these goals. Generally speaking, the graph structure can be collected in various ways. One direct strategy is to use prior knowledge. For instance, for data mining tasks, social networks \cite{DBLP:conf/kdd/QiuTMDW018} can be used to depict instance similarity. The label similarity can also be estimated using an external knowledge graph \cite{DBLP:conf/cvpr/LeeFYW18}.

However, we notice that there is no available graph structure in existing public PML datasets. Hence, we propose to construct the instance graph $\bm{A}^x\in\mathbb{R}^{n\times n}$ and the label graph $\bm{A}^y\in\mathbb{R}^{L\times L}$ in a data-driven way. To obtain $\bm{A}^x$, we adopt a monomial kernel \cite{DBLP:conf/cvpr/IscenTAFC17} as the distance metric and calculate a sparse instance affinity matrix $\bm{S}\in\mathbb{R}^{n\times n}$ with elements,
\begin{equation*}
\begin{split}
s_{ij} :=
\begin{cases}
    \max(\bm{x}_i^\top\bm{x}_j, 0)^\rho, \quad \text{if } i\ne j\text{ and } j\in\mathcal{N}_k(\bm{x}_i) \\
    0, \quad\quad\quad\quad\quad\quad \text{otherwise}
\end{cases}.
\end{split}
\end{equation*}
Here $\mathcal{N}_k(\bm{x}_i)$ saves the indices of $k$-nearest neighbors of $\bm{x}_i$. $\rho>0$ controls the weights of similarity. Then, we get a symmetrized instance graph by $\bm{A}^x=\bm{S}+\bm{S}^\top$.

Inspired by HALE \cite{DBLP:conf/kdd/LyuFL20}, we build the label graph $\bm{A}^y$ by means of the label co-occurrence statistics from the training data, whose elements are given by,
\begin{equation*}
\begin{split}
a^y_{ij} := \frac{\sum_{k=1}^n\mathbb{I}(y_{ki}=1 \wedge y_{kj}=1)}{\sum_{k=1}^n\mathbb{I}(y_{ki}=1) + \sum_{k=1}^n\mathbb{I}(y_{kj}=1)},
\end{split}
\end{equation*}
where $\mathbb{I}(\cdot)$ is an indicator function returning $1$ if the event is true and $0$ otherwise. It is worth noting that HALE simply calculates the average co-occurrence by dividing the instance number, which overlooks the existence of false-positive labels. In our method, we penalize it by the counts of instances that are assigned the labels as the candidates. In other words, if one label is assigned to too many instances, its co-occurrence statistics are less reliable.

Next, we impose the instance-level similarity property on the model output by means of a Laplacian regularizer,
\begin{equation}\label{instance-sim}
\begin{split}
\mathcal{J}_x(\hat{\bm{Y}})&=\sum_{i,j=1}^n a_{ij}^x||\frac{\hat{\bm{y}}_i}{\sqrt{d^x_{i}}}-\frac{\hat{\bm{y}}_j}{\sqrt{d^x_{j}}}||^2 =\trace(\hat{\bm{Y}}^\top\bm{L}_x\hat{\bm{Y}}),\\
\end{split}
\end{equation}
where $\hat{\bm{Y}}=[\hat{\bm{y}}_1,\ldots,\hat{\bm{y}}_n]^\top$ is the predicted label matrix. $\bm{L}_x=\bm{I}-\bm{D}_x^{-\frac{1}{2}}\bm{A}^x\bm{D}_x^{-\frac{1}{2}}$ is the Laplacian matrix of instance graph. $\bm{D}_x=\text{diag}[d_{1}^x,\ldots,d_{n}^x]$ is a diagonal matrix with $d_{i}^x=\sum_{j=1}^n a^x_{ij}$. The objective in Eq. (\ref{instance-sim}) enforces two data examples to have similar predicted labels if they have a large affinity score. Similarly, we can introduce the label-level similarity by,
\begin{equation*}
\begin{split}
\mathcal{J}_y(\hat{\bm{Y}})=\trace(\hat{\bm{Y}}\bm{L}_y\hat{\bm{Y}}^\top), \quad \bm{L}_y=\bm{I}-\bm{D}_y^{-\frac{1}{2}}\bm{A}^y\bm{D}_y^{-\frac{1}{2}}.\\
\end{split}
\end{equation*}

Finally, we involve a candidate consistency term to get the following objective,
\begin{equation*}
\begin{split}
\mathcal{J}(\hat{\bm{Y}},\bm{Y})=\frac{\eta}{2}||\hat{\bm{Y}}-\bm{Y}||_F^2+\frac{\alpha}{2}\mathcal{J}_x(\hat{\bm{Y}})+\frac{\beta}{2}\mathcal{J}_y(\hat{\bm{Y}}),
\end{split}
\end{equation*}
where $||\cdot||_F$ is the Frobenius norm. $\eta$, $\alpha$, and $\beta$ are all positive hyperparameters.

In summary, the main goal of our objective is to utilize the instance- and label-level similarity to recover reliable label confidences and address the problem of false positives, while simultaneously enjoying an extra performance improvement by exploiting label correlations.


\subsection{Deep Propagation Network}\label{deep-models}
If a linear model is adopted, the parameters can be directly obtained by optimizing $\mathcal{J}(f(\bm{X}),\bm{Y})$ with gradient descent. However, in deep models, it is not the case. When we update the model using a mini-batch of data, the instance-level manifold regularizer can involve unexpected gradient updating on the nearest neighbors of these data. Then, the training time will increase rapidly. To remedy this problem, we propose an efficient optimization framework for training a deep model from PML data. Specifically, we maintain a smooth pseudo-label matrix $\bm{Z}$, which minimizes the propagation error while being close to the model output $\hat{\bm{Y}}$. 
In the sequel, we elaborate the training procedure of our model. The training procedure is also summarized in Algorithm \ref{algorithm-details}. 

\subsubsection{Propagation Step}
At $t$-th epoch, we first propagate the label information from candidates $\bm{Y}$ and the model output $\hat{\bm{Y}}$ to the intermediate ones. In graph-based methods, the output label is required to be smooth enough. Therefore, we use the Frobenius norm as the measure to ensure the closeness of $\bm{Z}$ and $\hat{\bm{Y}}$. Formally, we estimate $\bm{Z}$ by,
\begin{equation}\label{Z-error}
\begin{split}
\min_{\bm{Z}}\mathcal{L}_{\text{LP}}(\bm{Z}) =\frac{1}{2}||\bm{Z}-\hat{\bm{Y}}||_F^2+\mathcal{J}(\bm{Z},\bm{Y}).
\end{split}
\end{equation}
We can derive its gradient by,
\begin{equation}\label{gradient}
\begin{split}
\frac{\partial \mathcal{L}_{\text{LP}}}{\partial\bm{Z}}=(1+\eta)\bm{Z}+\alpha\bm{L}_x\bm{Z}+\beta\bm{Z}\bm{L}_y-(\hat{\bm{Y}}+\eta\bm{Y}).
\end{split}
\end{equation}
which is followed by the matrix properties that $\frac{\partial||\bm{Z}-\bm{Y}||_F^2}{\partial\bm{Z}}=2(\bm{Z}-\bm{Y})$ and $\frac{\partial \text{tr}(\bm{Z}^\top \bm{L}_x\bm{Z})}{\partial\bm{Z}}=(\bm{L}_x+\bm{L}_x^\top)\bm{Z}$. We have $\bm{L}_x+\bm{L}_x^\top=2\bm{L}_x$ because $\bm{L}_x$ is symmetric. The optimal condition $\frac{\partial \mathcal{L}_{\text{LP}}}{\partial\bm{Z}}=\bm{0}$ typically requires solving a Sylvester system, which can be time-consuming. Therefore, we solve this problem by gradient descent $\bm{Z}=\bm{Z}-\gamma\frac{\partial \mathcal{L}_{\text{LP}}(\bm{Z})}{\partial\bm{Z}}$,
where $\gamma>0$ is the learning rate. Note that in MLL datasets, negative labels usually dominate, which makes all the pseudo-labels tend to concentrate at a small value. To cope with this issue, we further perform column-wise min-max normalization on $\bm{Z}_t$ to obtain a balanced output.

\subsubsection{Model Updating Step}\label{model_update}
Given an instance vector $\bm{x}$, we first feed it into a deep neural model $\hat{\bm{y}}=f(\bm{x};\theta)$. Then, we update the model parameters to fit the pseudo-labels,
\begin{equation}\label{deep-obj}
\begin{split}
\min_\theta \mathcal{L}_{\text{Deep}}(\hat{Y}, \bm{Z})=\frac{1}{n}\sum\nolimits_{i=1}^n l(\hat{\bm{y}}_i, \bm{z}_i),
\end{split}
\end{equation}
where $l(\cdot,\cdot)$ is the loss function. Recall that in Eq. (\ref{Z-error}), we enforce $\bm{Z}$ to be close to $\hat{\bm{Y}}$ using the Frobenius norm. Hence, a natural choice of loss is to reuse the mean squared error (MSE), which is adopted by most graph-based PML algorithms \cite{DBLP:conf/ijcai/Wang0ZZHC19,DBLP:conf/aaai/0009LG20,zml-pami-pml}. However, in PML tasks, the disambiguated labels can still be noisy and the MSE loss leads to heavy overfitting. This is largely overlooked by previous works.

To tackle this problem, we shall have a closer look at the risk function. While MSE penalize too much on prediction mistakes, other robust surrogate losses can be considered, such as mean absolute error (MAE). However, we notice that MAE violates our smooth prediction assumption. To this end, we adopt the widely-used binary cross-entropy loss (BCE) for multi-label learning,
\begin{equation*}
\begin{split}
l(\hat{\bm{y}}_i, \bm{z}_i)\!=\!\sum\nolimits_{j\!=\!1}^L \!-\!z_{ij}\log(\sigma(\hat{y}_{ij}))\!-\!(1\!-\!z_{ij})\log(1\!-\!\sigma(\hat{y}_{ij})),
\end{split}
\end{equation*}
where $\sigma(\cdot)$ is the sigmoid function. On one hand, BCE loss approximately grows linearly for largely negative values, which makes it less sensitive to outliers. On the other hand, BCE outputs relatively smooth output such that the smoothness is preserved.

In effect, we can go further on the risk. A plethora of works \cite{DBLP:conf/aaai/GhoshKS17,DBLP:conf/icml/FengK000S20} has demonstrated that BCE is still less robust since it is unbounded. Let us consider the probabilistic output $\sigma(\hat{\bm{y}}_i)$ of the deep model. Then the BCE loss is equivalent to a negative log loss on $\sigma(\hat{\bm{y}}_i)$. If the predicted probability of the label is close to zero, the loss can grow quickly to infinity. Previous works have suggested that bounded loss on the probabilistic output is theoretically more robust than unbounded ones. One common choice of bounded loss is the probabilistic MSE loss (PMSE), i.e. $l(\hat{\bm{y}}_i, \bm{z}_i)=||\sigma(\hat{\bm{y}}_i)-\bm{z}_i||_2^2$. In that case, the training criterion is consistent with our propagation error. Empirically, both PMSE and BCE loss perform well, and thus, we instantiate our model using BCE loss when comparing it with other methods. In Section \ref{loss-comparison}, we present a comparison between these losses, which clearly verifies our conjectures.

\paragraph{Fast Implementation.} To efficiently build the instance graph from training data, we employ a quick similarity search python package Faiss \cite{JDH17} to find the nearest neighbors. The sparse matrix multiplication $\bm{L}_x\bm{Z}$ can be efficiently implemented by the Scipy package. In Appendix A, we provide a detailed analysis of the complexity and convergence properties to show the efficacy of our method.

\renewcommand{\algorithmicrequire}{\textbf{Input:}}
\renewcommand{\algorithmicensure}{\textbf{Output:}}
\begin{algorithm}[!t]
  \caption{The pseudo-code of PLAIN}
  \label{algorithm-details}
    \begin{algorithmic}[1]
            \ENSURE The deep MLL network parameter $\theta$
            \REQUIRE Training dataset $\mathcal{D}$, instance graph $\bm{A}^x$, label graph $\bm{A}^y$, hyperparameters $\gamma$, $\alpha$, $\beta$ and $\eta$, iteration number $T$
            \STATE $\bm{Z}\leftarrow \bm{Y}$
            \STATE $\bm{L}_x\leftarrow\bm{I}-\bm{D}_x^{-\frac{1}{2}}\bm{A}^x\bm{D}_x^{-\frac{1}{2}}$
            \STATE $\bm{L}_y\leftarrow\bm{I}-\bm{D}_y^{-\frac{1}{2}}\bm{A}^y\bm{D}_y^{-\frac{1}{2}}$
      \FOR {epoch $t = 1,2,3,\ldots$}
                \FOR {batch $j = 1,2,\ldots$ in $\mathcal{D}$}
                    \STATE Sample a mini-batch data $(\bm{X}^j, \bm{Z}^j)$
                    \STATE $\hat{\bm{Y}}^j\leftarrow f(\bm{X}^j;\theta)$
                    \STATE $\theta\leftarrow\text{OPTIMIZE}(l(\hat{\bm{Y}}^j, \bm{Z}^j;\theta))$
                \ENDFOR
                \STATE $\hat{\bm{Y}} \leftarrow f(\bm{X};\theta)$
                \FOR {iter $i = 1,2,\ldots,T$}
                    \STATE $\bm{Z}\leftarrow\bm{Z}-\gamma\frac{\partial \mathcal{L}_{\text{LP}}(\bm{Z})}{\partial\bm{Z}}$
                \ENDFOR
            \ENDFOR
    \end{algorithmic}
\end{algorithm}

\begin{table*}[!t]
\centering
\normalsize
\addtolength{\tabcolsep}{-1.8pt}
\caption{Comparison of PLAIN with baselines on real-world datasets, where the best results are shown in bold face.}
\label{table_realworld_results}
\begin{threeparttable}
\begin{tabular}{c|ccccccc}
\cline{1-8}
\toprule
\textbf{Datasets} &\textbf{PLAIN} &\textbf{PML-NI} &\textbf{HALE} &\textbf{ML-KNN} &\textbf{GLOCAL}  &\textbf{PML-lc} &\textbf{PARTICLE} \\
\midrule
&\multicolumn{7}{c}{{Ranking Loss (the lower the better)}} \\
\midrule
Music-emotion     &\textbf{0.217$\pm$0.006}  &0.244$\pm$0.007&0.299$\pm$0.008   &0.365$\pm$0.010 &0.344$\pm$0.025    &0.310$\pm$0.010     &{0.261$\pm$0.007}  \\
Music-style       &\textbf{0.136$\pm$0.005}  &0.138$\pm$0.009&0.271$\pm$0.015   &0.229$\pm$0.010 &0.253$\pm$0.007    &0.269$\pm$0.016     &0.351$\pm$0.014  \\
Mirflickr         &\textbf{0.088$\pm$0.006}  &0.123$\pm$0.004&{0.143$\pm$0.009} &0.178$\pm$0.013 &0.212$\pm$0.003    &0.206$\pm$0.008     &0.213$\pm$0.011 \\
\midrule
          &\multicolumn{7}{c}{{Average Precision (the higher the better)}} \\
\midrule
Music-emotion     &\textbf{0.664$\pm$0.007}  &0.610$\pm$0.011&0.543$\pm$0.005    &0.505$\pm$0.010  &0.513$\pm$0.010      &0.539$\pm$0.013     &{0.626$\pm$0.011} \\
Music-style       &\textbf{0.742$\pm$0.012}  &0.740$\pm$0.013&{0.687$\pm$0.013}  &0.659$\pm$0.014  &0.645$\pm$0.007      &0.598$\pm$0.009     &0.621$\pm$0.016   \\
Mirflickr         &\textbf{0.847$\pm$0.012}  &0.792$\pm$0.004&{0.718$\pm$0.005}  &0.698$\pm$0.013  &0.679$\pm$0.007      &0.675$\pm$0.010     &0.690$\pm$0.012  \\
\midrule
          &\multicolumn{7}{c}{{Hamming Loss (the lower the better)}} \\
\midrule
Music-emotion     &\textbf{0.191$\pm$0.004}  &0.254$\pm$0.013&0.219$\pm$0.003     &{0.364$\pm$0.011} &0.219$\pm$0.001    &0.221$\pm$0.008     &{0.206$\pm$0.003}   \\
Music-style       &\textbf{0.112$\pm$0.003}  &0.157$\pm$0.010&{0.122$\pm$0.005}   &{0.844$\pm$0.010} &0.144$\pm$0.001    &0.162$\pm$0.006     &0.137$\pm$0.007   \\
Mirflickr         &{0.162$\pm$0.003}  &0.223$\pm$0.005&\textbf{0.156$\pm$0.003}   &0.217$\pm$0.006   &0.253$\pm$0.001    &0.193$\pm$0.006     &0.179$\pm$0.008  \\

\bottomrule
\end{tabular}
\end{threeparttable}
\end{table*}

\begin{table*}[!t]
\centering
\normalsize
\addtolength{\tabcolsep}{-3.4pt}
\caption{Win/tie/loss counts of PLAIN's performance against baselines on synthetic data sets (pairwise $t$-test at 0.05 significance level).}
\label{wintieloss}
\begin{tabular}{c|ccccccccc|c}
\cline{1-11}
\toprule
\textbf{Metrics}          &\textbf{Emotions} &\textbf{Birds}    &\textbf{Medical}  &\textbf{Image}    &\textbf{Corel5k}    &\textbf{Bibtext}  &\textbf{Eurlex-dc}  &\textbf{Eurlex-sm} & \textbf{NUS-WIDE} & \textbf{Sum} \\ \midrule
Ranking Loss&18/0/0&15/0/3&15/0/3&15/1/2&18/0/0&18/0/0&15/0/3&14/4/0&16/0/2&144/5/13\\
Average Precision&18/0/0&16/1/1&15/0/3&18/0/0&16/0/2&16/1/1&18/0/0&18/0/0&18/0/0&153/2/7\\
Hamming Loss&15/1/2&14/1/3&11/4/3&11/1/6&5/13/0&11/7/0&8/10/0&14/4/0&14/4/0&103/45/14   \\
\midrule
\textbf{Sum}      &51/1/2&45/2/7&41/4/9&44/2/8&39/13/2&45/8/1&41/10/3&46/8/0&48/4/2&400/52/34    \\ \bottomrule
\end{tabular}
\end{table*}

\section{Experiments}\label{experiments}
In this section, we present the main empirical results of PLAIN compared with state-of-the-art baselines. More experimental results are shown in Appendix B. 

\subsection{Datasets}
We employe a total of twelve synthetic as well as real-world datasets for comparative studies. Specifically, the synthetic ones are generated from the widely-used multi-label learning datasets. A total of nine benchmark multi-label datasets\footnote{\href{http://mulan.sourceforge.net/datasets-mlc.html}{\color{blue}{http://mulan.sourceforge.net/datasets-mlc.html}}} are used for synthetic PML datasets generation, including \textsf{Emotions}, \textsf{Birds}, \textsf{Medical}, \textsf{Image}, \textsf{Bibtex}, \textsf{Corel5K}, \textsf{Eurlex-dc}, \textsf{Eurlex-sm}, \textsf{NUS-WIDE}. Note that \textsf{NUS-WIDE} is a large-scale dataset, and we preprocess it as in \cite{DBLP:conf/kdd/LyuFL20}. For each data example in MLL datasets, we randomly select $r$ irrelevant labels and aggregate them with the ground-truth labels to obtain a candidate set. For example, given an instance $\bm{x}_i$ and its ground-truth label set $\tilde{S}_i$, we select $r$ labels from its complementary set $\mathcal{Y}-\tilde{S}_i$. If there are less than $r$ irrelevant labels, i.e. $|\mathcal{Y}-\tilde{S}_i|<r$, we simply set the whole training set $\mathcal{Y}$ as the candidate set for $\bm{x}_i$. Following the experimental settings in \cite{DBLP:conf/kdd/LyuFL20}, we choose $r\in\{1,2,3\}$, resulting in $27$ $(3\times9)$ synthetic datasets. Furthermore, we conducted experiments on three real-world PML datasets \cite{zml-pami-pml}, including \textsf{Music-emotion}, \textsf{Music-style}, \textsf{Mirflickr}. These three datasets are derived from the image retrieval task \cite{DBLP:conf/mir/HuiskesL08}, where the candidate labels are collected from web users and then further verified by human annotators to determine the ground-truth labels. 


\subsection{Baselines and Implementation Details}
We choose six benchmark methods for comparative studies, including two MLL methods ML-KNN \cite{DBLP:journals/pr/ZhangZ07}, GLOCAL \cite{DBLP:journals/tkde/ZhuKZ18}, and four state-of-the-art PML methods PML-NI \cite{DBLP:journals/pami/XieH22}, PARTICLE \cite{zml-pami-pml}, HALE \cite{DBLP:conf/kdd/LyuFL20}, and PML-lc \cite{DBLP:conf/aaai/XieH18}. In particular, PARTICLE and HALE are two advanced graph-based PML methods. PARTICLE \cite{zml-pami-pml} adopts an instance-level label propagation algorithm to disambiguate the candidate labels. HALE \cite{DBLP:conf/kdd/LyuFL20} regards the denoising procedure as a graph-matching procedure.

The parameter setups for the used methods are as follows. We fix $k=10$ for all $k$NN-based methods, including PLAIN. For the baselines, we fix or fine-tune the hyperparameters as the suggested configurations in respective papers.
For our PLAIN method, the deep model is comprised of three fully-connected layers. The hidden sizes are set as $[64, 64]$ for those datasets with less than $64$ labels, and $[256, 256]$ for those datasets with more than $64$ and less than $256$ labels. For the remaining datasets, we set the hidden sizes as $[512, 512]$. The trade-off parameters $\alpha$ and $\beta$ are hand-tuned from $\{0.001,0.01,0.1\}$. $\eta$ is selected from $\{0.1,1,10\}$. Following \cite{DBLP:conf/cvpr/IscenTAFC17}, we set $\rho=3$. We train our deep model via stochastic gradient descent and empirically set the learning rate as $0.01$ for both propagation and deep model training procedures. The number of maximum iterations is set as $T=200$ for small-scale datasets and $T=50$ for the large-scale NUS-WIDE dataset. Besides, weight decay is applied with a rate of $5e^{-5}$ to avoid overfitting.

For performance measure, we use three widely-used multi-label metrics, \textit{ranking loss}, \textit{average precision} and \textit{hamming loss} \cite{DBLP:journals/tkde/ZhangZ14}. Finally, we perform ten-fold cross-validation on each dataset and the mean metric values as well as the standard deviations are reported.

\begin{table*}[!t]
\centering
\small
\addtolength{\tabcolsep}{0.8pt}
\caption{Comparison of PLAIN with baselines on nine synthetic PML datasets ($r$ = 3), where the best results are shown in bold face.}
\label{table_main_results3}
\begin{threeparttable}
\begin{tabular}{c|ccccccc}
\cline{1-8}
\toprule
\textbf{Datasets} &\textbf{PLAIN} &\textbf{PML-NI} &\textbf{HALE} &\textbf{ML-KNN} &\textbf{GLOCAL}  &\textbf{PML-lc} &\textbf{PARTICLE} \\\midrule
&\multicolumn{7}{c}{{Ranking Loss (the lower the better)}} \\
\midrule
Emotions    &\textbf{0.158$\pm$0.024}  &0.214$\pm$0.029&{0.235$\pm$0.037}   &0.241$\pm$0.026 &0.322$\pm$0.053 &0.459$\pm$0.035 &0.259$\pm$0.019 \\
Birds       &0.205$\pm$0.036  &\textbf{0.187$\pm$0.035}&{0.271$\pm$0.061}   &0.304$\pm$0.048 &0.302$\pm$0.041 &0.321$\pm$0.021 &0.301$\pm$0.032 \\
Medical     &0.050$\pm$0.011  &\textbf{0.039$\pm$0.013}&0.169$\pm$0.025     &0.088$\pm$0.019 &0.068$\pm$0.005 &0.056$\pm$0.012 &0.100$\pm$0.021 \\
Image       &\textbf{0.190$\pm$0.015}  &0.289$\pm$0.018&{0.192$\pm$0.015}   &0.342$\pm$0.026 &0.264$\pm$0.010 &0.467$\pm$0.025 &0.315$\pm$0.073 \\
Corel5K     &\textbf{0.120$\pm$0.004}  &0.205$\pm$0.006&{0.259$\pm$0.011}   &0.139$\pm$0.007 &0.164$\pm$0.003 &{0.169$\pm$0.013} &{0.333$\pm$0.050}  \\
Bibtext     &\textbf{0.079$\pm$0.005}  &0.126$\pm$0.010&0.601$\pm$0.009     &0.232$\pm$0.006 &0.131$\pm$0.006 &0.342$\pm$0.005 &0.287$\pm$0.010 \\
Eurlex-dc   &0.033$\pm$0.002  &\textbf{0.030$\pm$0.003}&{0.079$\pm$0.002}   &0.086$\pm$0.012 &0.150$\pm$0.003 &0.071$\pm$0.013 &{0.061$\pm$0.023} \\
Eurlex-sm   &\textbf{0.027$\pm$0.001}  &0.028$\pm$0.002&{0.040$\pm$0.002}   &0.043$\pm$0.003 &0.071$\pm$0.003 &0.082$\pm$0.013 &0.053$\pm$0.002  \\
NUS-WIDE    &\textbf{0.211$\pm$0.003}  &{0.221$\pm$0.002}&{0.239$\pm$0.012}   &0.301$\pm$0.011 &0.311$\pm$0.020 &-&0.240$\pm$0.015 \\
\midrule
          &\multicolumn{7}{c}{{Average Precision (the higher the better)}} \\
\midrule
Emotions    &\textbf{0.812$\pm$0.043}  &0.754$\pm$0.016&{0.751$\pm$0.035}   &0.741$\pm$0.029 &0.647$\pm$0.030 &0.573$\pm$0.025 &0.745$\pm$0.024 \\
Birds       &\textbf{0.606$\pm$0.037}  &0.580$\pm$0.062&{0.505$\pm$0.054}   &0.453$\pm$0.052 &0.390$\pm$0.060 &0.388$\pm$0.035 &0.431$\pm$0.051 \\
Medical     &0.781$\pm$0.039  &\textbf{0.849$\pm$0.035}&{0.769$\pm$0.023}   &0.672$\pm$0.039 &0.751$\pm$0.009 &0.713$\pm$0.012 &0.720$\pm$0.044 \\
Image       &\textbf{0.774$\pm$0.020}  &0.668$\pm$0.018&{0.762$\pm$0.016}   &0.627$\pm$0.022 &0.692$\pm$0.008 &0.523$\pm$0.019 &0.689$\pm$0.096 \\
Corel5K     &\textbf{0.319$\pm$0.010}  &0.289$\pm$0.011&{0.231$\pm$0.008}   &0.251$\pm$0.007 &0.284$\pm$0.007 &{0.247$\pm$0.012} &{0.135$\pm$0.041}  \\
Bibtext     &\textbf{0.550$\pm$0.014}  &0.537$\pm$0.012&{0.353$\pm$0.010}   &0.306$\pm$0.006 &0.438$\pm$0.011 &0.283$\pm$0.010 &0.313$\pm$0.012 \\
Eurlex-dc   &\textbf{0.740$\pm$0.008}  &0.700$\pm$0.011&{0.673$\pm$0.006}   &0.603$\pm$0.012 &0.278$\pm$0.005 &0.602$\pm$0.012 &0.630$\pm$0.016 \\
Eurlex-sm   &\textbf{0.802$\pm$0.007}  &0.718$\pm$0.011&{0.739$\pm$0.013}   &0.761$\pm$0.012 &0.556$\pm$0.008 &0.582$\pm$0.013 &0.695$\pm$0.012 \\
NUS-WIDE    &\textbf{0.286$\pm$0.004}  &0.274$\pm$0.003&{0.230$\pm$0.007}   &0.171$\pm$0.015 &0.177$\pm$0.024 &- &0.206$\pm$0.017\\
\midrule
          &\multicolumn{7}{c}{{Hamming Loss (the lower the better)}} \\
\midrule
Emotions    &\textbf{0.223$\pm$0.033}   &0.455$\pm$0.059&{0.297$\pm$0.071}          &{0.607$\pm$0.209}          &0.281$\pm$0.015 &0.437$\pm$0.027 &0.233$\pm$0.018 \\
Birds       &{0.088$\pm$0.010}          &0.095$\pm$0.013&0.096$\pm$0.009            &\textbf{0.053$\pm$0.006}   &0.096$\pm$0.008 &0.132$\pm$0.012 &0.142$\pm$0.018 \\
Medical     &{0.019$\pm$0.001}          &\textbf{0.014$\pm$}0.002&0.017$\pm$0.001   &{0.022$\pm$0.002}          &0.028$\pm$0.001  &0.063$\pm$0.003 &0.024$\pm$0.002 \\
Image       &{0.207$\pm$0.009}          &0.273$\pm$0.028&\textbf{0.189$\pm$0.017}   &{0.753$\pm$0.005}          &0.237$\pm$0.009&0.443$\pm$0.014 &0.403$\pm$0.042 \\
Corel5K     &\textbf{0.009$\pm$0.000}   &0.011$\pm$0.000&{0.014$\pm$0.001}          &0.009$\pm$0.000            &0.009$\pm$0.000 &{0.019$\pm$0.002} &\textbf{0.009$\pm$0.000}  \\
Bibtext     &\textbf{0.013$\pm$0.000}   &0.015$\pm$0.000&0.016$\pm$0.001            &{0.014$\pm$0.000}          &0.015$\pm$0.001 &0.021$\pm$0.001 &0.017$\pm$0.000 \\
Eurlex-dc   &\textbf{0.002$\pm$0.001}   &0.010$\pm$0.001&{0.004$\pm$0.003}          &{0.009$\pm$0.005}          &\textbf{0.002$\pm$0.000} &0.013$\pm$0.005 &0.004$\pm$0.000 \\
Eurlex-sm   &\textbf{0.006$\pm$0.002}   &0.011$\pm$0.001& 0.008$\pm$0.005           &{0.012$\pm$0.001}          &0.011$\pm$0.000 &0.019$\pm$0.002 &0.010$\pm$0.001 \\
NUS-WIDE    &\textbf{0.021$\pm$0.000}   &0.035$\pm$0.002&0.031$\pm$0.008            &{0.049$\pm$0.016}          &0.022$\pm$0.001  &-&{0.290$\pm$0.006}\\
\bottomrule
\end{tabular}
\begin{tablenotes}
\footnotesize
\item[$\sharp$] - means over long time consumption, and thus, the experimental results are not reported.
\end{tablenotes}
\end{threeparttable}
\end{table*}

\subsection{Experimental Results}
Table \ref{table_realworld_results} reports the experimental results on real-world datasets. Table \ref{table_main_results3} lists the results on synthetic datasets, where the parameter is configured with $r=3$. The results of $r=1,2$ are consistent with those of $r=3$ and thus are omitted. To better understand the superiority of the proposed method, we summarize the win/tie/loss counts on all synthetic datasets in Table \ref{wintieloss}. 
We also report the statistical testing results in Appendix B for analyzing the relative performance amongst competing algorithms. 
Based on these results, we can draw the following observations:
\begin{itemize}
    \item The proposed PLAIN method significantly outperforms all other competing approaches. For example, on the Music-emotion and Mirflicker datasets, in terms of average precision, PLAIN achieved results superior to those of the best baselines by $\bm{6.07}$\% and $\bm{6.94}$\% respectively.
    \item According to Table \ref{wintieloss}, out of 486 statistical comparisons, six competing algorithms, and three evaluation metrics, PLAIN beats the competing methods in $\bm{82.30}$\% cases. In particular, PLAIN wins or ties the baselines in $\bm{95.68}$\% cases regarding average precision.
    \item ML-KNN and GLOCAL obtain inferior performance on many datasets, such as Medical and Eurlex-dc. The reason is that these tailored MLL models regard all the candidate labels as the ground-truth and thus tend to overfit the false-positive labels.
    \item Two PML methods PML-NI and HALE are strong baselines. In particular, PML-NI achieved the second-best results on most datasets. However, PLAIN performs significantly better than them, since their performance is restricted to the linear classifier.
    \item On large-scale dataset NUS-WIDE, PLAIN still outperforms other methods on all evaluation metrics. It demonstrates the effectiveness of PLAIN in handling high-volume of data.
\end{itemize}

\subsection{Further Analysis}\label{loss-comparison}

\begin{figure*}[t]
\begin{center}
    \subfigure[Parameter sensitivity of $\gamma$]{
        \includegraphics[width=0.275\linewidth]{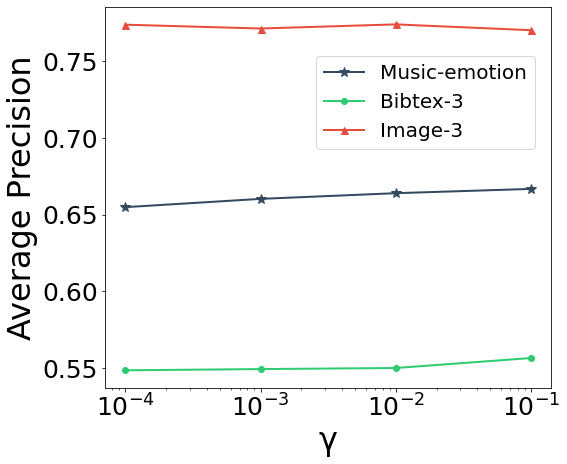}
    }
    \subfigure[Parameter sensitivity of $T$]{
        \includegraphics[width=0.275\linewidth]{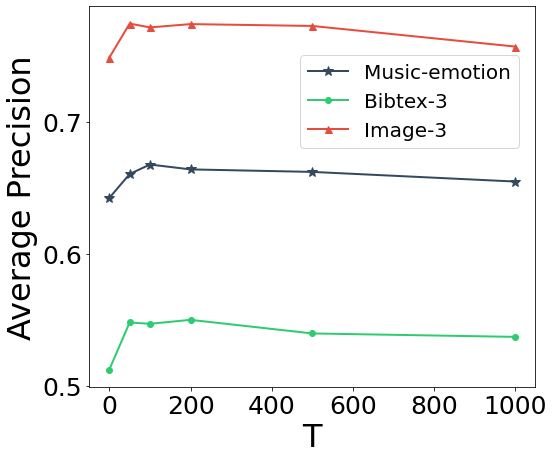}
    }
    \subfigure[Analysis of risk functions]{
        \includegraphics[width=0.275\linewidth]{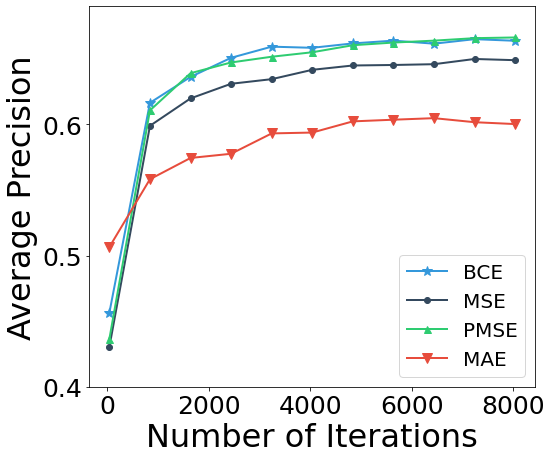}
    }
\vskip -0.1in
\end{center}
   \caption{Analysis of parameter sensitivities. (a) Changes in the performance of PLAIN as $\gamma$ changes; (b) Changes in the performance of PLAIN as $T$ changes; (c) Convergence curves of PLAIN with different risk functions on Music-emotion dataset. }
\label{fig:additional_exp}
\end{figure*}

\begin{figure}[t]
\begin{center}
    \includegraphics[width=0.48\linewidth]{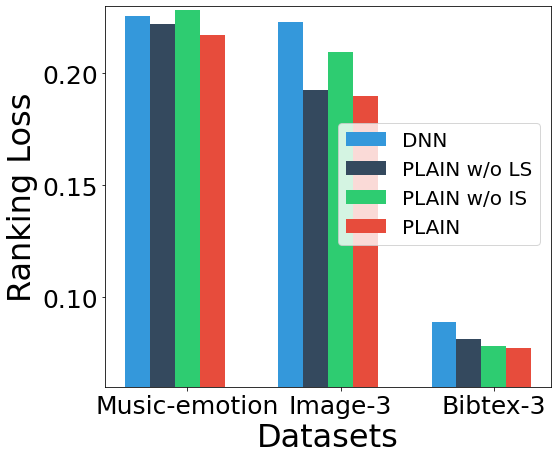}
    \includegraphics[width=0.48\linewidth]{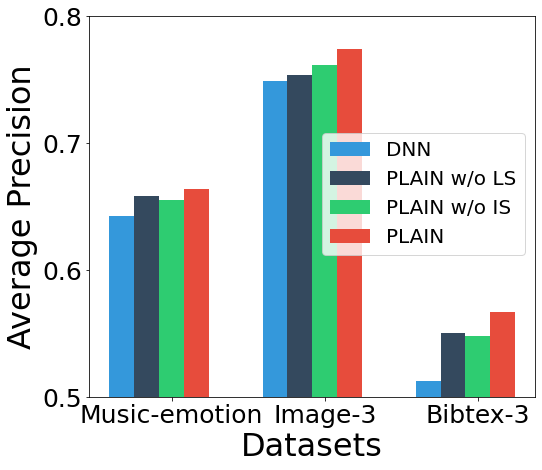}
\end{center}
   \caption{Ablation study on three datasets. We set $r=3$ for Image and Bibtex. Notably, the lower ranking loss is the better and so is the hamming loss. For Average Precision, the higher is the better. }
\label{fig:ablation}
\end{figure}

\paragraph{Complexity Analysis}

In Table \ref{runningtime}, we report the running time of different training stages of PLAIN, including the graph building step, label propagation step (one epoch) and the deep model training step (one epoch), on datasets with various scales. First of all, we can see that the instance graph can be fast built in a few microseconds on all datasets, with the help of Faiss package \cite{JDH17}. Moreover, by utilizing the sparse property of instance graph and the fast computation ability of GPUs, the time complexity of label propagation is at the same magnitude as that of deep model training. In particular, on those datasets with fewer labels, such as Music-emotion and Image, the propagation phase is as fast as the deep model traveling the whole training dataset. In Figure \ref{fig:additional_exp}, (a) and (b) show the parameter sensitivities of PLAIN to $T$ and $\gamma$. We can observe that PLAIN is stable with varying values of $T$ and $\gamma$. Specifically, PLAIN obtains high average precision even with small $T$ and small $\gamma$. But when $T=0$ (without label propagation), our PLAIN model degenerates to the trivial DNN model and thus, the performance degrades. It should also be noted that when $T$ is too large, the performance of PLAIN slightly drops because the model tends to overfit. In summary, the proposed PLAIN model requires only a few gradient steps in the label propagation phase and can be efficiently trained.

\paragraph{Comparison between Loss Functions}
As discussed in section \ref{model_update}, mean square error (MSE) is error-prone to the label noise. Hence, we conducted experiments to compare the performance of different loss functions, including MSE, MAE, BCE and probabilistic MSE (PMSE). The results are shown in Figure \ref{fig:additional_exp} (c). We can see that MAE demonstrates inferior performance and low convergence. Despite its robustness, it tends to produce harsh predictions and is opposite to the graph techniques which generate smooth outputs. MSE obtains good performance, but underperforms BCE and PMSE, since it easily overfits the false-positive labels. Finally, both BCE and PMSE demonstrate good empirical results. These observations suggest that a robust loss function is essential in successful learning from PML data.

\begin{table}
\addtolength{\tabcolsep}{1pt}
\centering
\small
\caption{Running time (\textbf{in microseconds}) of PLAIN on six datasets with various data scales, regarding the graph building, label propagation (one epoch), and deep model training (one epoch) stages.}
\label{runningtime}
\begin{tabular}{c|c|c|c}
\cline{1-4}
\toprule
\textbf{Datasets}  & \makecell[c]{\textbf{Graph}\\\textbf{Building}}&\makecell[c]{\textbf{Label}\\\textbf{Propagation}} & \makecell[c]{\textbf{Model}\\\textbf{Training}}
\\ \midrule
Image & 1.0  & 150.2  & 154.4 \\
Medical & 1.1  & 344.0  & 112.2 \\
Music-emotion & 5.6  & 449.6  & 480.7 \\
Bibtex & 26.7  & 2498.7  & 600.8 \\
Eurlex-sm & 355.7  & 8911.2  & 1902.2 \\
NUS-WIDE & 1829.3  & 17255.2  & 11251.6 \\
\bottomrule
\end{tabular}
\end{table}

\paragraph{Ablation Analysis}
An ablation study would be informative to see how different components of the proposed PLAIN contribute to the prediction performance. Thus, we compare PLAIN with three variants: 1) a pure deep neural network model (DNN) that treats the candidate labels as valid ones; 2) the PLAIN model without considering label-level similarity (PLAIN w/o LS); 3) the PLAIN model without considering instance-level similarity (PLAIN w/o IS). Figure \ref{fig:ablation} reports the results on one real-world dataset as well as two synthetic datasets. In general, both instance- and label-level graph regularizers improve the simple DNN model on most datasets and evaluation metrics. Then, by integrating the two regularizers, our PLAIN model consistently achieves the best performance. These results clearly verify the superiority of the proposed method.

\section{Conclusion}
In this work, we proposed a novel deep partial multi-label learning model PLAIN that copes with several limitations in existing graph-based PML methods. At each training epoch, we first involve both the instance- and label-level similarities to generate pseudo-labels via label propagation. We then train the deep model to fit the disambiguated labels. Moreover, we analyzed the time complexity and the robustness of PLAIN from both theoretical and empirical perspectives. Comprehensive experiments on synthetic datasets as well as real-world PML datasets clearly validate the efficiency, robustness, and effectiveness of our proposed method. Our work demonstrates that combining the graph and deep models together leads to high expressiveness as well as good disambiguation ability. 
We hope our work will increase attention toward a broader view of graph-based deep PML methods.

%
%
\section*{Acknowledgments}
This work is supported by the National Key R\&D Program of China (No. 2022YFB3304100) and by the Zhejiang University-China Zheshang Bank Co., Ltd. Joint Research Center. Gengyu Lyu is supported by the National Key R\&D Program of China (No. 2022YFB3103104) and the China Postdoctoral Science Foundation (No. 2022M720320).

%

\bibliographystyle{named}
\bibliography{references}

\newpage

\appendix

\setcounter{figure}{3}
\setcounter{table}{4}

\section{Theoretical Analysis}

\subsection{Complexity Analysis}
It is worth noting that the propagation step can be efficient. First of all, we have disentangled the propagation step from the deep model training, which allows us to perform simple algebra operations. Moreover, since $\bm{Z}$ is close to both $\hat{\bm{Y}}$ and $\bm{Y}$, either one can be used to initialize $\bm{Z}$ and only a few gradient steps are required to get a satisfactory solution.

Theoretically, in Eq. (3), calculating the term $\bm{Z}\bm{L}_y$ requires $\mathcal{O}(nL^2)$ time. To compute $\bm{L}_x\bm{Z}$, since the instance graph is highly sparse, we need $\mathcal{O}(n*\text{nnz}(\bm{L}_x))=\mathcal{O}(n*(n+\text{nnz}(\bm{A}_x)))$, where $\text{nnz}(\cdot)$ denotes the number of non-zero entries in a matrix. In our data-driven graph-construction case, we have $\text{nnz}(\bm{A}_x)=nk$. Therefore, at each epoch, we require $\mathcal{O}(nT(L^2+n(k+1)))$ time in the graph propagation procedure, where $T$ is the number of gradient steps. In our implementation, we use the widely-used matrix computation package Scipy to perform sparse matrix multiplication $\bm{L}_x\bm{Z}$. The remaining calculations, including the deep model training, are done on GPUs using PyTorch. According to our empirical analysis, the graph propagation procedure is indeed as fast as the deep training and be efficiently computed. 

\subsection{Convergence Analysis}
It is also of interest to see the convergence performance of PLAIN method. Take the Probabilistic MSE risk as an example, the overall objective can be formulated as follows,
\begin{equation*}
\begin{split}
\mathcal{L}(\theta,\bm{Z})=||\bm{Z}-\sigma(f(\bm{X};\theta))||_F^2+\mathcal{J}(\bm{Z},\bm{Y})
\end{split}
\end{equation*}
with a constant factor ignored on the deep model objective. We further assume that the gradient descent algorithm is used as the $\text{OPTIMIZE}(\cdot)$ function in Algorithm 1 and the learning rates are selected properly. The following Proposition ensures the convergence of our method and proves the value of the objective function is non-increasing at each updating step.

\begin{table}[!t]
\small
\addtolength{\tabcolsep}{-2.6pt}
\centering
\caption{Characteristics of the experimental datasets. The last three PML datasets are real-world datasets.}
\label{data-stats}
\begin{threeparttable}
\begin{tabular}{ccccccc}
\cline{1-7}
\toprule
\textbf{Datasets}      &\textbf{EXPs}\tnote{\dag}      &\textbf{FEAs}        &\textbf{CLs}          &\textbf{M-GT}           &\textbf{A-GT}  & \textbf{DOM}\\ \midrule
Emotions      &593        &72           &6             &3               &1.87   & music\\
Birds         &645        &260          &19            &6               &1.01   & audio  \\
Medical       &978        &1,449        &45            &3               &1.25   & text  \\
Image         &2,000      &294          &5             &3               &1.23   & image  \\
Corel5K       &5,000      &499          &374           &5               &3.52   & image  \\
Bibtext       &7,395      &1,836        &159           &28              &2.40   & text  \\
Eurlex-dc     &19,348     &5,000        &412           &7               &1.29   & text  \\
Eurlex-sm     &19,348     &5,000        &201           &12              &2.21   & text  \\
NUS-WIDE\tnote{\dag}&133,441&500        &81            &20              &1.76   & image  \\ \midrule
Music-emotion &6,833      &98           &11            &7               &2.42   & music  \\
Music-style   &6,839      &98           &10            &4               &1.44   & music  \\
Mirflickr     &10,433     &100          &7             &5               &1.77   & image  \\ \bottomrule
\end{tabular}
\begin{tablenotes}
\footnotesize
\item[\ddag]  For each PML dataset, the number of examples (EXPs), features (FEAs), class labels (CLs), the maximum number of ground-truth labels (M-GT), the average number of ground-truth labels (A-GT) and its corresponding domain (DOM) are recorded.
\item[\dag] The original number of instances is 269,648 but some of them are unlabeled w.r.t the 81 class labels, thus we only utilized the remaining 133,441 instances to conduct experiments.
\end{tablenotes}
\end{threeparttable}
\end{table}

\begin{proposition}
The updating rules listed in Algorithm 1 guarantee finding a stationary point of the objective $\mathcal{L}(\theta,\bm{Z})$.
\end{proposition}
\begin{proof}
Let $(\theta_t,\bm{Z}_t)$ be the solution at the $t$-th epoch and $\mathcal{L}(\theta_t,\bm{Z}_t)$ be the corresponding objective function value. In the propagation step, $\theta_t$ is fixed and the remaining objective, i.e. Eq. (2), is a standard convex optimization problem. Thus, our gradient descent algorithm guarantees to decrease the objective function value. Next, in the deep model training step, $\bm{Z}_t$ is fixed and the objective in Eq. (4) is non-convex. However, the gradient descent algorithm still decreases the objective until a first-order stationary point is achieved. We can conclude that,
\begin{equation*}
\begin{split}
\mathcal{L}(\theta_t,\bm{Z}_t)\le\mathcal{L}(\theta_t,\bm{Z}_{t-1})\le\mathcal{L}(\theta_{t-1},\bm{Z}_{t-1}). \end{split}
\end{equation*}
Consequently, the objective function $\mathcal{L}(\theta,\bm{Z})$ will monotonically decrease and finally converge to a stationary point.
\end{proof}

\section{Additional Experiments and Implementation Details}\label{experiments}
In this section, we present additional empirical results. All the computations are carried out on a workstation with an Intel E5-2680 v4 CPU, a Quadro RTX 5000 GPU, and 500GB main memory running the Linux platform. We summarize the characteristics of the used datasets in Table \ref{data-stats}. 

\begin{table}
\normalsize
\caption{Friedman statistics $\tau_F$ in terms of each evaluation
metric.}
\centering
\label{table-f-statics}
\begin{tabular}{c|c|c}
\cline{1-3}
\toprule
\textbf{Metrics}  &\textbf{$\tau_F$} &\textbf{Critical Value} \\ \midrule
Ranking Loss &31.88        & \multirow{3}*{\shortstack{2.15\\ (Methods: 7, Data sets: 30)} }\\
Average Precision &49.26 &\\
Hamming Loss   &2.87   &\\
\bottomrule
\end{tabular}
\end{table}

\begin{figure*}[!ht]
\centering
\includegraphics[width=0.95\linewidth]{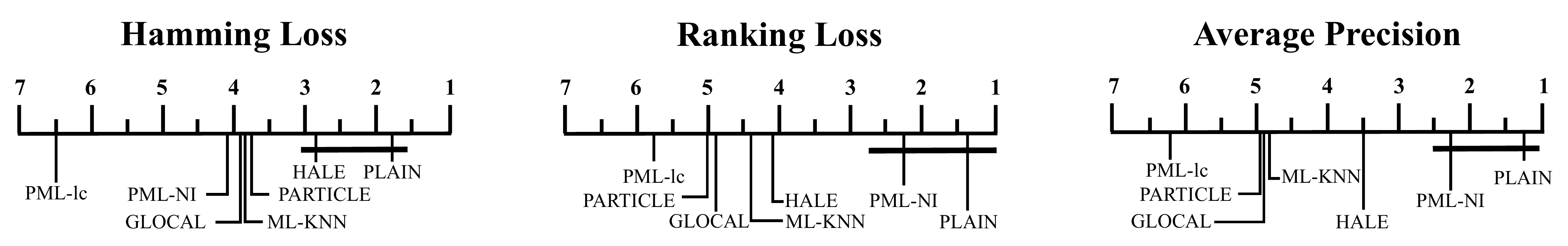}
\caption{Comparison of PLAIN (control algorithm) against six comparing algorithms with the Bonferroni-Dunn test. Algorithms not connected with PLAIN in the CD diagram are considered to have a significantly different performances from the control algorithm. (CD = 1.47 at 0.05 significance level)}
\label{BDTest}
\end{figure*}

\subsection{Statistical Tests}
Following previous works \cite{DBLP:conf/kdd/LyuFL20,zml-pami-pml}, to comprehensively evaluate the superiority of PLAIN, we further utilized \textit{Friedman test} \cite{DBLP:journals/jmlr/Demsar06} as the statistical test to analyze the relative performance among the competing methods. The experimental results are reported in Table \ref{table-f-statics}. At a 0.05 significance level, the null hypothesis of indistinguishable performance of PLAIN among all competing methods is clearly rejected. Subsequently, we employ the \textit{Bonferroni-Dunn} test as the posthoc test by regarding PLAIN as the control approach. Figure \ref{BDTest} reports the CD diagrams on each evaluation metric, where the average rank of the competing approaches is marked along the axis. The performance of the control method and one learning method is deemed to be significantly different if their average ranks differ by at least one CD. According to Figure \ref{BDTest}, we can observe that PLAIN achieves highly superior results to other baseline methods.


\subsection{Further Analysis}\label{loss-comparison}


\paragraph{Convergence Curve} We also conducted experiments on Music-emotion to show the convergence curve of the objective function. The results in Figure \ref{fig:appendix_exp} (a) further support our proposition. Besides, the experimental results in Figure 2 (c) also agree with our theoretical findings. When the binary cross-entropy loss is used, our deep model is trained in an asymmetric fashion, and then there is no overall objective for our algorithm. However, as shown in Figure 2 (c), our method still tends to converge to a good solution.

\begin{figure}[t]
\begin{center}
    \subfigure[Convergence curve]{
        \includegraphics[width=0.45\linewidth]{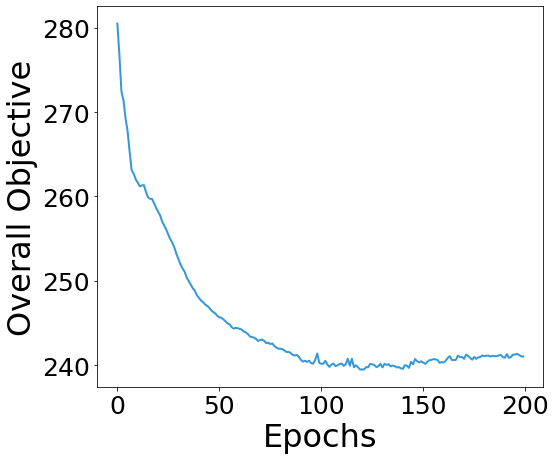}
    }
    \subfigure[Parameter sensitivity of $T$]{
        \includegraphics[width=0.45\linewidth]{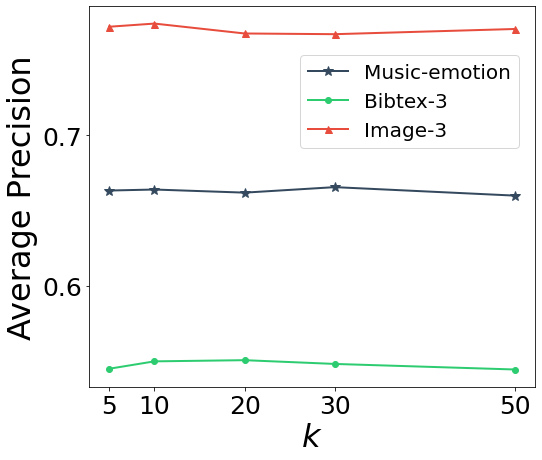}
    }
\end{center}
   \caption{(a) The convergence curve of objective function value $\mathcal{L}(\theta,\bm{Z})$ on Music-emotion dataset. (b) Changes in the performance of PLAIN as the number of nearest neighbors $k$ changes. }
\label{fig:appendix_exp}
\end{figure}



\begin{table}[!t]
\centering
\normalsize
\addtolength{\tabcolsep}{-3pt}
\caption{Comparison of PLAIN with two representative deep PML methods APML and PML-MT on three real-world datasets. The best results are shown in bold face.}
\label{deep-comp}
\begin{threeparttable}
\begin{tabular}{c|ccc}
\cline{1-4}
\toprule
\textbf{Datasets}  & \textbf{PLAIN} & \textbf{APML} & \textbf{PML-MT$^\sharp$}
\\ \midrule
&\multicolumn{3}{c}{{Ranking Loss (the lower the better)}} \\ \midrule
Music-emotion     & \textbf{0.215$\pm$0.005} & 0.242$\pm$0.007 & 0.236$\pm$0.007\\
Music-style      & \textbf{0.134$\pm$0.007} & 0.145$\pm$0.006 & -\\
Mirflickr         & \textbf{0.088$\pm$0.008} & 0.124$\pm$0.014 & 0.126$\pm$0.016 \\\midrule
&\multicolumn{3}{c}{{Average Precision (the higher the better)}} \\ \midrule
Music-emotion     & \textbf{0.664$\pm$0.007}  & 0.621$\pm$0.013 & 0.627$\pm$0.010\\
Music-style      & \textbf{0.743$\pm$0.009} & 0.732$\pm$0.010 & -\\
Mirflickr         & \textbf{0.845$\pm$0.013} & 0.777$\pm$0.027 & 0.807$\pm$0.034 \\\midrule
&\multicolumn{3}{c}{{Hamming Loss (the lower the better)}} \\ \midrule
Music-emotion     & \textbf{0.192$\pm$0.003}  & 0.200$\pm$0.004 & 0.207$\pm$0.011\\
Music-style      & \textbf{0.114$\pm$0.004} & 0.115$\pm$0.002 & -\\
Mirflickr         & \textbf{0.166$\pm$0.003} & 0.170$\pm$0.003 & 0.173$\pm$0.011\\
\bottomrule
\end{tabular}
\begin{tablenotes}
\footnotesize
\item[$\sharp$] The results of PML-MT on Music-style are not reported in the corresponding paper.
\end{tablenotes}
\end{threeparttable}
\end{table}

\paragraph{Comparing with Deep PML Methods}

To demonstrate the effectiveness of PLAIN, we further compare with two state-of-the-art deep PLL methods APML \cite{DBLP:journals/corr/abs-1909-06717} and PML-MT \cite{DBLP:journals/kbs/YanLF21} on three real-world datasets according to the reported results in their papers. Since their experimental settings are different from our work, we follow \cite{DBLP:journals/corr/abs-1909-06717,DBLP:journals/kbs/YanLF21} to split the datasets to 80\% training and 20\% testing. We then rerun the proposed PLAIN on these datasets and the experimental results are reported in Table \ref{deep-comp}. We can observe that PLAIN outperforms APML and PML-MT on all datasets and all evaluation metrics. For instance, on the Music-emotion, Music-style, and Mirflicker datasets, in terms of average precision, PLAIN achieves results superior to those of the best-competing methods by $\textbf{5.90}$\%, $\textbf{1.50}$\% and $\textbf{4.71}$\% respectively. These results demonstrate that the graph technique improves the disambiguation ability of deep PML models and further confirms the effectiveness of our proposed PLAIN.


\paragraph{Comparing with Two-stage Variant}
To show the superiority of our end-to-end design, we further experiment with a two-stage variant of PLAIN which first propagates labels on the graph until converges and then trains the deep model to fit the pseudo-labels. The results are shown in \ref{tbl:two-stage}. We found the two-stage method underperforms PLAIN because it fails to disambiguate labels in the first stage, making deep model overfits. In contrast, PLAIN incorporates accurate model prediction into propagation for refined pseudo-labels. We will add complete results in the revision.

\begin{table}[!h]
\centering
\caption{Comparison of PLAIN with two-stage learning variant.}
\label{tbl:two-stage}
\addtolength{\tabcolsep}{1pt}
\begin{tabular}{ccccc}
\toprule
Metrics & Rloss$\downarrow$ & AP$\uparrow$ & Hloss$\downarrow$\\
\midrule
Two-Stage &0.236&0.637&0.201\\
PLAIN &\textbf{0.217} & \textbf{0.664} &\textbf{0.191}\\
\bottomrule
\end{tabular}
\end{table}

\paragraph{Parameter Analysis of $k$}\label{sensitivity_k}
Furthermore, we study the sensitivity analysis of PLAIN with respect to the critical parameter $k$. In Figure \ref{fig:appendix_exp} (b), we show the performance of PLAIN changes as $k$ increases from 5 to 50 on three datasets. We can observe that PLAIN is robust in terms of the parameter $k$ and thus, we empirically fix $k = 10$ in our experiments.

\begin{table*}[!t]
\centering
\caption{ List of notations.}
\label{tbl:notations}
\begin{tabular}{p{0.1\linewidth}p{0.6\linewidth}}
\toprule
\multicolumn{2}{c}{\textbf{Data}} \\
\midrule
$n$ &  Number of data examples\\
$d,L$ & Dimensions of features and labels \\
$r$ & Average number of false-positive labels in datasets  \\
$\mathcal{D}$ & Training dataset \\
$\mathcal{X},\mathcal{Y}$ &  Feature space and label space\\
$S,\tilde{S}$ & Candidate label set and true label set \\
$\bm{y},\tilde{\bm{y}},\hat{\bm{y}}$ & Candidate label vector, true label vector, and predicted label vector\\
$\bm{Y},\hat{\bm{Y}}$ & Candidate label matrix and predicted label matrix \\
$\bm{Z}$ & Pseudo-Label matrix \\
\midrule
\multicolumn{2}{c}{\textbf{Algorithm}} \\
\midrule
$\mathcal{N}_k(\bm{x})$ & Indices of $k$-nearest neighbors of $\bm{x}$\\
$\bm{S}$ & Sparse instance affinity matrix \\
$\bm{A}^x,\bm{A}^y$ & Instance and label graphs \\
$\bm{D}_x,\bm{D}_x$ &  Diagonal instance and label degree matrices\\
$\bm{L}_x,\bm{L}_y$ & Instance and label Laplacian matrices \\
$\bm{I}$ & Identity matrix \\
\midrule
\multicolumn{2}{c}{\textbf{Function}} \\
\midrule
$f(\cdot)$ & Neural network \\
$\theta$ & Parameters of the neural network \\
$\mathcal{J}_x,\mathcal{J}_y$ & Instance and label graph regularizers \\
$\mathcal{J}$ & Overall propagation objective\\
$\sigma(\cdot)$ & Sigmoid function \\
$l(\cdot)$ & Loss function \\
$\mathcal{L}_\text{LP},\mathcal{L}_\text{Deep}$ & Overall objectives of propagation and deep model training \\
$\text{tr}(\cdot)$ & Trace of a matrix \\
$\mathbb{I}(\cdot)$ & Indicator function for equivalance testing \\
\midrule
\multicolumn{2}{c}{\textbf{Hyper-parameter}} \\
\midrule
$k$ & Number of nearest neighbors \\
$\rho$ & Weights of graph similarity \\
$\eta,\alpha,\beta$ & Loss weighting factors of candidate consistency, instance regularizer, and label regularizer \\
$\gamma$ & Learning rate of gradient descent for label propagation \\
$T$ & Number of propagation gradient updating steps \\

\bottomrule
\end{tabular}
\end{table*}

\section{Notations and Terminology}
The notations are summarized in Table \ref{tbl:notations}.

\end{document}